\documentclass{ecai} 

\usepackage{latexsym}
\usepackage{amssymb}
\usepackage{amsmath}
\usepackage{amsthm}
\usepackage{booktabs}
\usepackage{enumerate}
\usepackage{graphicx}
\usepackage{color}
\usepackage{breqn}

\newtheorem{theorem}{Theorem}

\newcommand{\BibTeX}{B\kern-.05em{\sc i\kern-.025em b}\kern-.08em\TeX}

\newtheorem{remark}{Remark}
\def\R{\mathbb{R}}
\providecommand{\norm}[1]{\ensuremath{\left\lVert#1\right\rVert }}
\providecommand{\mnorm}[1]{\ensuremath{\left\lvert#1\right\rvert}}
\newtheorem{assumption}{\bfseries Assumption}
\DeclareRobustCommand{\bigO}{%
  \text{\usefont{OMS}{cmsy}{m}{n}O}%
}
\newcommand\smallO{
  \mathchoice
    {{\scriptstyle\mathcal{O}}}
    {{\scriptstyle\mathcal{O}}}
    {{\scriptscriptstyle\mathcal{O}}}
    {\scalebox{.7}{$\scriptscriptstyle\mathcal{O}$}}
  }

\def\E{\mathbb{E}}
\def\Ek{\mathbb{E}_{k}}
\def\Ezk{\mathbb{E}_{\zeta_k}}
\def\V{\mathbb{V}_{\zeta_k}}
\providecommand{\condexp}[1]{\Ezk\ensuremath{\left[#1\right]}}
\providecommand{\totexp}[1]{\Ek\ensuremath{\left[#1\right]}}

\providecommand{\var}[1]{\V\ensuremath{\left[#1\right]}}

\begin{document}


\begin{frontmatter}


\paperid{1531} 


\title{A Methodology Establishing Linear Convergence \\ of Adaptive Gradient Methods under PL Inequality}


\author[A]{\fnms{Kushal}~\snm{Chakrabarti}\orcid{0000-0002-6747-8709}\thanks{Corresponding Author. Email: chakrabarti.k@tcs.com.}}
\author[A,B]
{\fnms{Mayank}~\snm{Baranwal}\orcid{0000-0001-9354-2826}}

\address[A]{Tata Consultancy Services Research, Mumbai, India}
\address[B]{Department of Systems \& Control Engineering, Indian Institute of Technology, India}


\begin{abstract}%
  Adaptive gradient-descent optimizers are the standard choice for training neural network models. Despite their faster convergence than gradient-descent and remarkable performance in practice, the adaptive optimizers are not as well understood as vanilla gradient-descent. A reason is that the dynamic update of the learning rate that helps in faster convergence of these methods also makes their analysis intricate. Particularly, the simple gradient-descent method converges at a {\em linear} rate for a class of optimization problems, whereas the practically faster adaptive gradient methods lack such a theoretical guarantee. The Polyak-\L ojasiewicz (PL) inequality is the weakest known class, for which {\em linear} convergence of gradient-descent and its momentum variants has been proved. Therefore, in this paper, we prove that AdaGrad and Adam, two well-known adaptive gradient methods, converge {\em linearly} when the cost function is smooth and satisfies the PL inequality. Our theoretical framework follows a simple and unified approach, applicable to both batch and stochastic gradients, which can potentially be utilized in analyzing {\em linear} convergence of other variants of Adam. %
\end{abstract}

\end{frontmatter}


\section{Introduction}
\label{sec:intro}

In this paper, we consider the problem of minimizing a possibly non-convex {\em objective function} $f: \R^d \to \R$,
\begin{align}
    \min_{x \in \R^d} f(x). \label{eqn:opt_1}
\end{align}
Among other applications, non-convex optimization appears in training neural network models. It is a standard practice to use adaptive gradient optimizers for training such models. Compared to gradient-descent, these methods have been observed to converge faster and do not require line search to determine the learning rates. 

AdaGrad~\cite{duchi2011adaptive} is possibly the earliest adaptive gradient optimizer. To address the gradient accumulation problem in AdaGrad, the Adam algorithm~\cite{kingma2014adam} was proposed. Adam and its variants have been widely used to train deep neural networks in the past decade. Despite the success of these adaptive gradient optimizers, we lack an understanding of why these methods work so well in practice. The theory of convergence of the adaptive optimizers has not been completely developed. Since these methods work better than simple gradient-descent or its momentum variants, it is natural to expect a similar or better convergence guarantee of adaptive optimizers compared to gradient-descent. One such theoretical difference between these two classes of methods is {\em linear} convergence. Specifically, gradient-descent and its accelerated variants, such as the Nesterov accelerated gradient or Heavy-Ball method, are known to exhibit {\em linear} convergence for a class of cost functions~\cite{necoara2019linear}. On the other hand, most of the adaptive gradient optimizers lack such a guarantee.

{\em Linear} convergence guarantees of gradient-descent, its accelerated variants, coordinate descent, and AdaGrad-Norm \cite{xie2020linear} (among the adaptive gradient methods) have been proved for the class of smooth and possibly non-convex cost functions that satisfy the Polyak-\L ojasiewicz (PL) inequality~\cite{karimi2016linear}. The PL inequality is the weakest condition among others, such as strong convexity, essential strong convexity, weak strong convexity, and restricted secant inequality, that leads to {\em linear} convergence of gradient-descent and its accelerated variants to the solution of~\eqref{eqn:opt_1}~\cite{karimi2016linear}. The objective functions in standard machine learning problems like linear regression and logistic regression satisfy the PL inequality. For solving over-parameterized non-linear equations,~\cite{liu2020toward} establishes a relation between PL inequality and the condition number of the tangent kernel, and argued that sufficiently wide neural networks generally satisfy the PL inequality. Motivated by {\em linear} convergence guarantees of gradient-descent and its aforementioned variants under the PL condition and the applicability of PL inequality on a set of machine learning problems, we investigate {\em linear} convergence of AdaGrad and Adam under the PL inequality. 

To present our results, we define the following notations and make a set of assumptions as stated below. Let the gradient of $f$ evaluated at $x \in \R^d$ be denoted by $\nabla f(x) \in \R^d$, and its $i$-th element be denoted by $\nabla_i f(x)$ for each dimension $i \in \{1,\ldots,d\}$.

\begin{assumption} \label{assump_1}
The minimum $f_*$ of $f$ exists and is finite, i.e., $\mnorm{\min_{x \in \R^d} f(x)} < \infty$.
\end{assumption}

\begin{assumption} \label{assump_2}
$f$ is twice differentiable over its domain $\R^d$ and is $L$-smooth, i.e., $\exists L > 0$ such that $\norm{\nabla f(x) - \nabla f(y)} \leq L \norm{x-y}$ for all $x,y \in \R^d$.
\end{assumption}

\begin{assumption} \label{assump_3}
$f$ satisfies the Polyak-\L ojasiewicz (PL) inequality, i.e., $\exists l > 0$ such that $\frac{1}{2}\norm{\nabla f(x)}^2 \geq l (f(x) - f_*)$ for all $x \in \R^d$.
\end{assumption}

Assumption~\ref{assump_3} has been justified in the preceding paragraph.
Assumption~\ref{assump_1}-\ref{assump_2} are standard in the literature of gradient-based optimization. An implication of Assumption~\ref{assump_2} is that~\cite{bottou2010large}
\begin{align}
    f(y) - f(x) \hspace{-0.2em} \leq \hspace{-0.2em} (y-x)^{\top} \nabla f(x) + \frac{L}{2} \norm{x-y}^2, \, \forall x,y \in \R^d. \label{eqn:smooth}
\end{align}

Next, we review the AdaGrad~\cite{duchi2011adaptive} algorithm and its existing convergence guarantees. AdaGrad and Adam are iterative methods. At each iteration $k=0,1,\ldots$, AdaGrad maintains an estimate $x_k$ of the solution of~\eqref{eqn:opt_1} and an auxiliary state $y_k$ that determines the learning rate. $x_k$ and $y_k$ are updated as
\begin{subequations}
\begin{align}
    y_{k+1,i} & = y_{k,i} + \mnorm{\nabla_i f(x_k)}^2, \label{eqn:y_adagrad}\\
    x_{k+1,i} & = x_{k,i} - h \frac{\nabla_i f(x_k)}{\sqrt{y_{k+1,i}} + \epsilon}. \label{eqn:x_adagrad}
\end{align}
\end{subequations}
Here, $h > 0$ is a stepsize, and $\epsilon > 0$ is a small-valued parameter to avoid division by zero.
The original AdaGrad paper~\cite{duchi2011adaptive} proves $\bigO(\sqrt{K})$ regret bound for convex $f$, where $K$ is the number of iterations, assuming the stochastic gradients and the estimates of minima are uniformly bounded. Under the same boundedness assumptions, for strongly convex $f$, AdaGrad's regret bound analysis~\cite{duchi2010adaptive} imply a $\bigO(\log(K)/K)$ convergence rate~\cite{chen2018sadagrad}. Following~\cite{duchi2011adaptive}, there are several works on the convergence of AdaGrad. The notable among them include~\cite{li2019convergence, defossez2020simple, ward2020adagrad, wang2023convergence}. Among them,~\cite{wang2023convergence} most recently proved $\bigO(1/K)$ convergence rate of AdaGrad for non-convex $f$, without the bounded gradient assumption. To the best of our knowledge, {\em linear} convergence of AdaGrad has not been proved in the literature. On the other hand, the gradient-descent algorithm is known to have {\em linear} convergence when $f$ is {\em smooth} and satisfies the PL inequality~\cite{karimi2016linear}. So, we ask the question of whether the adaptive gradient-descent methods, such as AdaGrad and Adam, can have guaranteed {\em linear} convergence for the same class of optimization problems.

Among the variants of AdaGrad,~\cite{xie2020linear} proves {\em linear} convergence of the AdaGrad-Norm (norm version of AdaGrad) in the stochastic settings under strong convexity of $f$ and in the deterministic setting under PL inequality. The AdaGrad-Norm algorithm, described as
\begin{align*}
    y_{k+1} & = y_{k} + \norm{\nabla f(x_k)}^2, \\
    x_{k+1} & = x_{k} - h \frac{\nabla f(x_k)}{\sqrt{y_{k+1}} + \epsilon},
\end{align*}
differs from the coordinate-wise update in AdaGrad~\eqref{eqn:y_adagrad}-\eqref{eqn:x_adagrad}, which is more commonly used in practice~\cite{wang2023convergence}. So, the analysis in~\cite{xie2020linear} does not trivially extend to AdaGrad. Furthermore, in the deterministic settings, if $y_0$ is initialized as a small value, AdaGrad-Norm is shown to have a {\em sublinear} rate until a finite number of iterations where $y_k$ crosses a certain threshold~\cite{xie2020linear}. In the same settings, we prove that AdaGrad has a {\em linear} rate for each iteration $k \geq 0$. \cite{mukkamala2017variants} proved a $\bigO(\log K)$ regret bound of the SC-AdaGrad algorithm, with $y_{k+1,i}$ replacing $\sqrt{y_{k+1,i}}$ in the denominator of AdaGrad~\eqref{eqn:x_adagrad} and a coordinate-wise $\epsilon_k$, when $f$ is strongly convex and the gradients and the estimates of minima are uniformly bounded. The SAdaGrad algorithm, a double-loop algorithm with AdaGrad in its inner loop, has a convergence rate of $\bigO(1/K)$ for weakly strongly convex $f$~\cite{chen2018sadagrad}. For two-layer networks with a positive definite kernel matrix,~\cite{wu2022adaloss} proves {\em linear} convergence of the AdaLoss algorithm if the width of the hidden layer is sufficiently large. Compared to the aforementioned works on the analysis of AdaGrad and its variants, we prove {\em linear} convergence of the original AdaGrad~\eqref{eqn:y_adagrad}-\eqref{eqn:x_adagrad} when $f$ is $L$-smooth and satisfies the PL inequality.

Next, we review the Adam algorithm~\cite{kingma2014adam} and its existing convergence results. At each iteration $k=0,1,\ldots$, Adam maintains an estimate $x_k$ of the solution of~\eqref{eqn:opt_1} and two moment estimates $\mu_k$ and $\nu_k$ that determine the learning rate. These estimates  are updated as
\begin{subequations}
\begin{align}
    \mu_{k+1,i} & = \beta_{1k} \mu_{k,i} + (1-\beta_{1k}) \nabla_i f(x_k), \label{eqn:mu_adam} \\
    \nu_{k+1,i} & = \beta_2 \nu_{k,i} + (1-\beta_2) \mnorm{\nabla_i f(x_k)}^2, \label{eqn:nu_adam} \\
    x_{k+1,i} & = x_{k,i} - h \frac{\mu_{k+1,i}}{\sqrt{\nu_{k+1,i}} + \epsilon}. \label{eqn:x_adam}
\end{align}
\end{subequations}
Here, $\beta_{1k},\beta_2 \in [0,1)$ are two algorithm parameters.
For non-convex optimization in stochastic settings,~\cite{zaheer2018adaptive, guo2021novel} have proved $\bigO(1/{\sqrt{K}})$ convergence rate of Adam and a family of its variants when the stochastic gradients are uniformly bounded.~\cite{guo2022novel} proved $\bigO(1/K)$ convergence rate of double-loop algorithms with Adam-style update in its inner loop in the stochastic settings under PL inequality and bounded stochastic gradients. Without the bounded gradients assumption,~\cite{zhang2022adam} proved $\bigO(\log(K)/\sqrt{K})$ convergence rate of Adam for non-convex optimization in stochastic settings. \cite{wang2019sadam} proved a $\bigO(\log K)$ regret bound of the SAdam algorithm, with $\nu_{k+1,i}$ replacing $\sqrt{\nu_{k+1,i}}$ in the denominator of Adam~\eqref{eqn:x_adam} and a vanishing $\epsilon_k$, when $f$ is strongly convex and the gradients and the estimates of minima are uniformly bounded.~\cite{barakat2021convergence} proved exponential convergence of a continuous-time version of Adam under the PL inequality. However,~\cite{barakat2021convergence} assumes that the continuous-time version of the momentum parameters $\beta_1$ and $\beta_2$ converge to one as the stepsize $h$ converge to zero. Also, the exponential convergence guarantee in~\cite{barakat2021convergence} does not extend to discrete-time, as discussed by the authors. From the above literature review, we note that a {\em linear} convergence of Adam in discrete-time has not been proved. We prove {\em linear} convergence of the discrete-time Adam algorithm~\eqref{eqn:mu_adam}-\eqref{eqn:x_adam} in the deterministic settings when $f$ is $L$-smooth and satisfies the PL inequality.

We aim to present a unified proof sketch for AdaGrad and Adam in this paper, which can potentially be utilized for variants of Adam in the future. Thus, our convergence analyses of AdaGrad and Adam both follow a similar approach. It is well known that the difficulty in obtaining a convergence rate, or even showing asymptotic convergence, of the adaptive gradient methods stems from both numerator and denominator of $x_{k+1} - x_{k}$ being dependent on the gradient (and its history) and, thus, the first order component in~\eqref{eqn:smooth} does not admit a straightforward descent direction, unlike the vanilla gradient-descent. The existing analyses of adaptive methods usually tackle this challenge by adding and subtracting a surrogate term, leading to an additional complicated term in the error. 

Our approach involves splitting the denominator in $x_{k+1,i} - x_{k,i}$ into two cases, depending on whether the denominator is always less than one or crosses one after a finite iterations. In the former case, the analysis from~\eqref{eqn:smooth} becomes equivalent to that of the vanilla gradient-descent or its momentum-variant, without any adaptive stepsize. In the latter case, the challenge is that the denominator of $x_{k+1,i} - x_{k,i}$ is not apriori bounded and, hence, the argument used in gradient-descent does not trivially apply (ref.~\eqref{eqn:fdiff_2} later). To address this challenge in this case, our analysis involves two broader steps. In the first step, the condition on the denominator allows us to prove decrement in the cost function over the iterations, simply for AdaGrad and after a more involved argument for Adam due to its momentum in the numerator of $x_{k+1,i} - x_{k,i}$. In the next step, we use the telescopic sum, followed by a simple argument to prove the boundedness of the denominator. Once we have a bounded denominator, the analysis again becomes similar to the former case. Finally, in the general case, where the numerator of $x_{k+1,i} - x_{k,i}$ crosses one only along a subset of the dimensions $\{1,\ldots,d\}$ and stays less than one along the rest of the dimensions, the linear rate is proved by the weakest (maximum) of the linear rates obtained in the previous two cases.

To not make our exposition too burdensome, we analyze AdaGrad and Adam first in the determinsitic setting, followed by analysis of AdaGrad with stochastic gradients. Due to space limit, we do not present the analysis of Adam with stochastic gradients, which can be done by following the technique presented in this paper.

The key contributions of our paper are as follows.
\begin{itemize}
    \item We prove {\em linear} convergence of the discrete-time AdaGrad~\eqref{eqn:y_adagrad}-\eqref{eqn:x_adagrad} and the Adam algorithm~\eqref{eqn:mu_adam}-\eqref{eqn:x_adam} in the deterministic setting when $f$ is $L$-smooth and satisfies the PL inequality. To the best of our knowledge, the {\em linear} convergence guarantees of these two methods do not exist in the literature. We do not require convexity or boundedness of the gradients. Our analysis of {\em linear} convergence is existential in the sense that we prove {\em linear} convergence of the optimizers, but do not characterize the coefficient of convergence explicitly in terms of the problem and algorithm parameters.
    \item Our analyses of AdaGrad and Adam follow a similar approach, with the difference lying in the additional arguments for Adam due to its momentum term of the gradients. We split the convergence analysis into two broad cases depending on the magnitude of the denominator in $x_{k+1,i} - x_{k,i}$ and leverage the condition for each such case to present a simple convergence analysis. Thus, our proofs of AdaGrad and Adam are simpler than the existing works. Moreover, the presented proof methodology, as outlined earlier, provides us with a unified recipe for analyzing other adaptive gradient methods. The novelty of our technique lies in addressing the denominator of $x_{k+1,i} - x_{k,i}$, which is not bounded apriori, in a simpler way, by proving its boundedness.
    \item Considering the practical scenario, we prove {\em linear} convergence in {\em expectation} of AdaGrad with stochastic gradients to a neighborhood of the minima. The analysis with stochastic gradients follow the same outline described above in the deterministic setting.
\end{itemize}

Although we prove existence of a {\em linear} convergence rate for AdaGrad and Adam, a limitation of our analysis is that it does not characterize the exact coefficient by explicitly relating it with the algorithm parameters and the problem constants. However, at least for AdaGrad, such explicit relation can be obtained by a closer look into our analysis, which is difficult for Adam due to the additional momentum term. Due to limited space, we do not present the analysis of Adam with stochastic gradients, as it can be done by following the steps in the proofs of Theorem~\ref{thm:adam} and Theorem~\ref{thm:stoc}.
\section{Linear convergence of AdaGrad}
\label{sec:adagrad}

\begin{theorem} \label{thm:adagrad}
Consider the AdaGrad algorithm in~\eqref{eqn:y_adagrad}-\eqref{eqn:x_adagrad} with initialization $y_0 = 0_d$ and $x_0 \in \R^d$ and the parameter $\epsilon \in (0,1)$. If Assumptions~\ref{assump_1}-\ref{assump_3} hold, then there exists $\overline{h} > 0$ such that if the step size $0 < h < \overline{h}$, then $(f(x_{k+1}) - f_*) \leq \rho (f(x_k) - f_*)$ for all $k \geq 0$, where $\rho \in (0,1)$.
\end{theorem}

\begin{proof}
Consider an arbitrary iteration $k \geq 0$. Under Assumption~\ref{assump_2}, from~\eqref{eqn:smooth} we have
\begin{align*}
    f(x_{k+1}) - f(x_k) & \leq (x_{k+1} - x_k)^{\top} \nabla f(x_k) + \frac{L}{2} \norm{x_{k+1} - x_k}^2.
\end{align*}
Upon substituting above from~\eqref{eqn:x_adagrad},
\begin{align}
    & f(x_{k+1}) - f(x_k) \nonumber \\
    & \leq \sum_{i=1}^d \left( - h \frac{\mnorm{\nabla_i f(x_k)}^2}{\sqrt{y_{k+1,i}} + \epsilon} + \frac{L h^2}{2} \frac{\mnorm{\nabla_i f(x_k)}^2}{\mnorm{\sqrt{y_{k+1,i}} + \epsilon}^2}\right). \label{eqn:fdiff_1}
\end{align}

\noindent {\bf Case-I}: First, we consider the case when $y_{k+1,i} > (1-\epsilon)^2$ for all $i \in \{1,\ldots,d\}$ for all $k \geq T$, where $T < \infty$. Recall that $\{y_{k,i}\}$ is a non-decreasing sequence for all $i$. Consider any iteration $k \geq T$. Then, $\mnorm{\sqrt{y_{k+1,i}} + \epsilon}^2 > \sqrt{y_{k+1,i}} + \epsilon$. From~\eqref{eqn:fdiff_1}, then we have
\begin{align}
    f(x_{k+1}) - f(x_k) 
    & \leq - h \left(1 - \frac{L h}{2}\right) \sum_{i=1}^d \frac{\mnorm{\nabla_i f(x_k)}^2}{\sqrt{y_{k+1,i}} + \epsilon}. \label{eqn:fdiff_2}
\end{align}
If $0 < h < \frac{2}{L}$, from above we have $f(x_{k+1}) \leq f(x_k)$. Since $\{f(x_k): k\geq T\}$ is a decreasing sequence and, under Assumption~\ref{assump_1}, bounded from below by $f_*$, the sequence converges with $\lim_{k \to \infty} f(x_k) < \infty$. Next, upon summation from $t=T$ to $t=k$ on both sides of~\eqref{eqn:fdiff_2}, due to telescopic cancellation on the L.H.S.,
\begin{align*}
    f(x_{k+1}) - f(x_T) \leq - h \left(1 - \frac{L h}{2}\right) \sum_{t=T}^k \sum_{i=1}^d \frac{\mnorm{\nabla_i f(x_t)}^2}{\sqrt{y_{t+1,i}} + \epsilon},
\end{align*}
which means
\begin{align*}
    & 0 \leq \lim_{k \to \infty} h \left(1 - \frac{L h}{2}\right) \sum_{t=T}^k \sum_{i=1}^d \frac{\mnorm{\nabla_i f(x_t)}^2}{\sqrt{y_{t+1,i}} + \epsilon} \\
    & \leq \lim_{k \to \infty} (f(x_T) - f(x_{k+1})) \leq f(x_T) - f_*. 
\end{align*}
So, $\lim_{k \to \infty} \sum_{t=T}^k \frac{\mnorm{\nabla_i f(x_t)}^2}{\sqrt{y_{t+1,i}} + \epsilon}$ is bounded, which implies that 
\begin{align}
    \lim_{k \to \infty} \frac{\mnorm{\nabla_i f(x_k)}^2}{\sqrt{y_{k+1,i}} + \epsilon} = 0, \, \forall i. \label{eqn:seq_adagrad}
\end{align}
Thus, either $\lim_{k \to \infty} \mnorm{\nabla_i f(x_k)} = 0$ or $\lim_{k \to \infty} \sqrt{y_{k+1,i}} = \infty$. Consider the case $\lim_{k \to \infty} \sqrt{y_{k+1,i}} = \infty$. From~\eqref{eqn:seq_adagrad}, we have $\mnorm{\nabla_i f(x_k)}^2 = \smallO(\sqrt{y_{k+1,i}} + \epsilon)$, which implies $\frac{\mnorm{\nabla_i f(x_k)}}{\sqrt{y_{k+1,i}} + \epsilon} = \smallO((\sqrt{y_{k+1,i}} + \epsilon)^{-0.5})$. So, if $\lim_{k \to \infty} \sqrt{y_{k+1,i}} = \infty$, we have $\lim_{k \to \infty} \frac{\mnorm{\nabla_i f(x_k)}}{\sqrt{y_{k+1,i}} + \epsilon} = 0$.
Since both $\lim_{k \to \infty} \frac{\mnorm{\nabla_i f(x_k)}^2}{\sqrt{y_{k+1,i}} + \epsilon}$ $= 0$ and $\lim_{k \to \infty} \frac{\mnorm{\nabla_i f(x_k)}}{\sqrt{y_{k+1,i}} + \epsilon} = 0$, it is possible only if $\lim_{k \to \infty} \mnorm{\nabla_i f(x_k)} = 0$. So, we have proved that $\lim_{k \to \infty} \mnorm{\nabla_i f(x_k)} = 0$ is true. Under Assumption~\ref{assump_3}, then we have $\lim_{k \to \infty} f(x_k) = f_*$. The above argument further shows that $\lim_{k \to \infty} \sqrt{y_{k+1,i}} = \infty$ is possible only in the trivial case where the function is at the minimum point. In the non-trivial case, therefore, $\sqrt{y_{k+1,i}} + \epsilon$ is bounded. Then, for $0 < h < \frac{2}{L}$, $\exists M \in (0,\infty)$ such that $\sqrt{y_{k+1,i}} + \epsilon \leq M \,  \forall i$. From~\eqref{eqn:fdiff_2},
\begin{align*}
    f(x_{k+1}) - f(x_k) & \leq - \frac{h}{M} \left(1 - \frac{L h}{2}\right) \norm{\nabla f(x_k)}^2.
\end{align*}
Under Assumption~\ref{assump_3}, from above we have
\begin{align*}
    f(x_{k+1}) - f(x_k) & \leq - \frac{h}{M} \left(1 - \frac{L h}{2}\right) 2l (f(x_k) - f_*).
\end{align*}
Upon defining $c_1 = h \left(1 - \frac{L h}{2}\right) \frac{2l}{M}$, we rewrite the above as
\begin{align*}
    f(x_{k+1}) - f(x_k) & \leq - c_1 (f(x_k) - f_*),
\end{align*}
which means
\begin{align}
    f(x_{k+1}) - f_* & \leq (1 - c_1) (f(x_k) - f_*), \, \forall k \geq T. \label{eqn:conv_adagrad_1}
\end{align}
Moreover, $0 < h < \frac{2}{L}$ implies that $c_1 > 0$ and $0 < h < \frac{M}{2l}$ implies that $c_1 < 1$. So, $(1-c_1) \in (0,1)$ for $0 < h < \min\{\frac{2}{L}, \frac{M}{2l}\}$. 

\noindent {\bf Case-II}: Next, we consider the case when $y_{k+1,i} \leq (1-\epsilon)^2$ for all $i \in \{1,\ldots,d\}$ for all $k \geq 0$. Then, $\sqrt{y_{k+1,i}} + \epsilon \leq 1$. Also, $y_0 = 0_d$ and~\eqref{eqn:y_adagrad} implies that $y_{k,i} \geq 0 \, \forall k,i$.
From~\eqref{eqn:fdiff_1}, then we have
\begin{align*}
    f(x_{k+1}) - f(x_k) & \leq - h \left(1 - \frac{Lh}{2\epsilon^2}\right) \norm{\nabla f(x_k)}^2.
\end{align*}
Under Assumption~\ref{assump_3}, from above we have
\begin{align*}
    f(x_{k+1}) - f(x_k) & \leq - h \left(1 - \frac{L h}{2 \epsilon^2}\right) 2l (f(x_k) - f_*).
\end{align*}
Upon defining $c_2 = h \left(1 - \frac{L h}{2 \epsilon^2}\right) 2l$, we obtain that
\begin{align}
    f(x_{k+1}) - f_* & \leq (1 - c_2) (f(x_k) - f_*),\, \forall k \geq 0. \label{eqn:conv_adagrad_2}
\end{align}
Moreover, $(1-c_2) \in (0,1)$ for $0 < h < \min\left\{\frac{2\epsilon^2}{L}, \frac{1}{2l}\right\}$. 

\noindent {\bf Case-III}: Finally, we consider the case when, for each $i \in \{1,\ldots,d\}$, $\exists T_i \in [0, \infty)$ such that $y_{k+1,i} \leq (1-\epsilon)^2$ for $k < T_i$ and $y_{k+1,i} > (1-\epsilon)^2$ for $k \geq T_i$. Then, $\exists T = \max \{T_1,\ldots,T_d\}$ such that $y_{k+1,i} > (1-\epsilon)^2$ for all $i \in \{1,\ldots,d\}$ for all $k \geq T$. For $k \geq T$, the analysis in Case-I directly applies. For $k < T$, from the analysis in Case-I- and Case-II, it follows that
\begin{align}
    f(x_{k+1}) - f_* & \leq \max\{(1 - c_1),(1-c_2)\} (f(x_k) - f_*), \label{eqn:conv_adagrad_3}
\end{align}
where $\max\{(1 - c_1),(1-c_2)\}$ for $0 < h < \min\{\frac{2}{L}, \frac{M}{2l}, \frac{2\epsilon^2}{L}, \frac{1}{2l}\}$. Since $\epsilon < 1$, $\frac{2\epsilon^2}{L} < \frac{2}{L}$.

We conclude that, for $0 < h < \min\{\frac{2 \epsilon^2}{L}, \frac{M}{2l}, \frac{1}{2l}\}$, 
\begin{align*}
    (f(x_{k+1}) - f_*) \leq \max\{(1 - c_1),(1-c_2)\} (f(x_k) - f_*), \, \forall k \geq 0,
\end{align*}
where $\max\{(1 - c_1),(1-c_2)\} \in (0,1)$.
The proof is complete with $\rho = \max\{(1 - c_1),(1-c_2)\}$ and $\overline{h} = \min\{\frac{2 \epsilon^2}{L}, \frac{M}{2l}, \frac{1}{2l}\}$.
\end{proof}

Since $\rho \in (0,1)$, Theorem~\ref{thm:adagrad} implicates that the sub-optimality gap $(f(x_k) - f_*)$ of the AdaGrad algorithm~\eqref{eqn:y_adagrad}-\eqref{eqn:x_adagrad} {\em linearly} converges to zero at the worst-case rate $\rho$, for small enough stepsize $h$.

\section{Linear convergence of Adam}
\label{sec:adam}

\begin{theorem} \label{thm:adam}
Consider the Adam algorithm in~\eqref{eqn:mu_adam}-\eqref{eqn:x_adam} with initialization $\mu_0 = \nu_0 = 0_d$ and $x_0 \in \R^d$ and the parameters $\beta_{1k}, \beta_2 \in [0,1)$, $\epsilon \in (0,1)$ . If Assumptions~\ref{assump_1}-\ref{assump_3} hold, then $\exists \overline{\beta_k} \in (0,1)$ such that for $\beta_{1k} \in [0,\overline{\beta_k})$ there exists $\overline{h} > 0$ such that if the stepsize $0 < h < \overline{h}$, then $(f(x_{k+1}) - f_*) \leq \rho (f(x_k) - f_*)$ for all $k \geq 0$ where $\rho \in (0,1)$.
\end{theorem}

\begin{proof}
Under Assumption~\ref{assump_2}, from~\eqref{eqn:smooth} we have
\begin{align*}
    f(x_{k+1}) - f(x_k) & \leq (x_{k+1} - x_k)^{\top} \nabla f(x_k) + \frac{L}{2} \norm{x_{k+1} - x_k}^2.
\end{align*}
Upon substituting above from~\eqref{eqn:x_adam},
\begin{align}
    & f(x_{k+1}) - f(x_k) \nonumber \\
    & \leq \sum_{i=1}^d \left( - h \frac{\mu_{k+1,i} \nabla_i f(x_k)}{\sqrt{\nu_{k+1,i}} + \epsilon} + \frac{L h^2}{2} \frac{\mnorm{\mu_{k+1,i}}^2}{\mnorm{\sqrt{\nu_{k+1,i}} + \epsilon}^2}\right). \label{eqn:fadam_1}
\end{align}

\noindent {\bf Case-I}: First, we consider the case $\nu_{k+1,i} > (1-\epsilon)^2$ for all $i \in \{1,\ldots,d\}$ for all $k \geq T$, where $T < \infty$. Then, $\mnorm{\sqrt{\nu_{k+1,i}} + \epsilon}^2 > \sqrt{\nu_{k+1,i}} + \epsilon$. From~\eqref{eqn:fadam_1}, then we have
\begin{align}
    & f(x_{k+1}) - f(x_k) \nonumber \\
    & \leq - h \sum_{i=1}^d \left(\frac{\mu_{k+1,i} \nabla_i f(x_k)}{\sqrt{\nu_{k+1,i}} + \epsilon} - \frac{L h}{2} \frac{\mnorm{\mu_{k+1,i}}^2}{\sqrt{\nu_{k+1,i}} + \epsilon}\right). \label{eqn:case_1}
\end{align}
Upon substituting $\mu_{k+1,i}$ above from~\eqref{eqn:mu_adam} and rearranging the terms in the numerator,
\begin{align}
    & f(x_{k+1}) - f(x_k) \nonumber \\
    &\leq\!- h\!\sum_{i=1}^d\!\frac{(1\!-\!\beta_{1k}) (1\!-\!h\frac{L}{2}(1\!-\!\beta_{1k})) \mnorm{\nabla_i f(x_k)}^2 \!-\! h \frac{L}{2} \beta_{1k}^2 \mnorm{\mu_{k,i}}^2}{\sqrt{\nu_{k+1,i}} + \epsilon} \nonumber \\
    & \!- h\!\sum_{i=1}^d\!\frac{\beta_{1k} (1\!-\!hL(1\!-\!\beta_{1k})) \mu_{k,i} \nabla_i f(x_k)}{\sqrt{\nu_{k+1,i}} + \epsilon} \nonumber \\
    &=\! - h\!\sum_{i=1}^d\!\frac{(1\!-\!\theta_k)(1\!-\!\beta_{1k}) (1\!-\!h\frac{L}{2}(1\!-\!\beta_{1k})) \mnorm{\nabla_i f(x_k)}^2 \!}{\sqrt{\nu_{k+1,i}} + \epsilon} \nonumber \\
    &\! - h\!\sum_{i=1}^d\!\frac{\theta_k (1\!-\!\beta_{1k}) (1\!-\!h\frac{L}{2}(1\!-\!\beta_{1k})) \mnorm{\nabla_i f(x_k)}^2}{\sqrt{\nu_{k+1,i}} + \epsilon} \nonumber \\
    & - h \sum_{i=1}^d \frac{- h \frac{L}{2} \beta_{1k}^2 \mnorm{\mu_{k,i}}^2 + \beta_{1k} (1-hL(1-\beta_{1k})) \mu_{k,i} \nabla_i f(x_k)}{\sqrt{\nu_{k+1,i}} + \epsilon}, \nonumber
\end{align}
for any $\theta_k \in (0,1)$. We define two set of indices $I = \{i \, | i\in \{1,\ldots,d\}, \mu_{k,i} \neq 0 \}$ and its complement $I'$. Then, we rewrite the above inequality as
\begin{align}
    & f(x_{k+1}) - f(x_k) \nonumber \\
    &\leq \!- h\!\sum_{i \in I}\!\frac{(1\!-\!\theta_k)(1\!-\!\beta_{1k}) (1\!-\!h\frac{L}{2}(1\!-\!\beta_{1k})) \mnorm{\nabla_i f(x_k)}^2\!}{\sqrt{\nu_{k+1,i}} + \epsilon} \nonumber \\
    &\!- h\!\sum_{i \in I}\!\frac{\theta_k (1\!-\!\beta_{1k}) (1\!-\!h\frac{L}{2}(1\!-\!\beta_{1k})) \mnorm{\nabla_i f(x_k)}^2}{\sqrt{\nu_{k+1,i}} + \epsilon} \nonumber \\
    & - h \sum_{i \in I} \frac{- h \frac{L}{2} \beta_{1k}^2 \mnorm{\mu_{k,i}}^2 + \beta_{1k} (1-hL(1-\beta_{1k})) \mu_{k,i} \nabla_i f(x_k)}{\sqrt{\nu_{k+1,i}} + \epsilon} \nonumber \\
    & - h \sum_{i \in I'} \frac{(1-\beta_{1k}) (1-h\frac{L}{2}(1-\beta_{1k})) \mnorm{\nabla_i f(x_k)}^2}{\sqrt{\nu_{k+1,i}} + \epsilon}. \label{eqn:fadam_2}
\end{align}
For $\beta_{1k} < 1$ and $h < \frac{2}{L(1-\beta_{1k})}$, the last term on the R.H.S. in~\eqref{eqn:fadam_2} is negative. So, for $i \in I$, we want to show that
\begin{align}
    & \theta_k (1\!-\!\beta_{1k}) (1\!-\!h\frac{L}{2}(1\!-\!\beta_{1k})) \mnorm{\nabla_i f(x_k)}^2 \nonumber \\
    & \geq h \frac{L}{2} \beta_{1k}^2 \mnorm{\mu_{k,i}}^2 + \beta_{1k} (1\!-\!hL(1\!-\!\beta_{1k})) \mnorm{\mu_{k,i} \nabla_i f(x_k)}, \label{eqn:fadam_3}
\end{align}
for some $\theta_k \in (0,1)$ and some $h > 0$, which would imply that $f(x_{k+1}) - f(x_k) \leq 0$, from~\eqref{eqn:fadam_2}. Consider any $i \in I$.
We note that~\eqref{eqn:fadam_3} is equivalent to
\begin{align*}
    & \theta_k (1-\beta_{1k}) \mnorm{\nabla_i f(x_k)}^2 - \beta_{1k} \mnorm{\mu_{k,i} \nabla_i f(x_k)} \geq h \frac{L}{2} \beta_{1k}^2 \mnorm{\mu_{k,i}}^2 \\ & + h \frac{L}{2} (1-\beta_{1k}) \left(\theta_k (1-\beta_{1k}) \mnorm{\nabla_i f(x_k)}^2 - 2\beta_{1k} \mnorm{\mu_{k,i} \nabla_i f(x_k)} \right),
\end{align*}
which is implied if
\begin{align}
    & \theta_k (1-\beta_{1k}) \mnorm{\nabla_i f(x_k)}^2 - \beta_{1k} \mnorm{\mu_{k,i} \nabla_i f(x_k)} \geq h \frac{L}{2} \beta_{1k}^2 \mnorm{\mu_{k,i}}^2 \nonumber \\ & + h \frac{L}{2} (1-\beta_{1k}) \left(\theta_k (1-\beta_{1k}) \mnorm{\nabla_i f(x_k)}^2 - \beta_{1k} \mnorm{\mu_{k,i} \nabla_i f(x_k)} \right). \label{eqn:fadam_4}
\end{align}

For the case $\mnorm{\mu_{k,i}} = \infty$, we choose $\beta_{1k} = 0$ and 
$h < \frac{2}{L}$ so that~\eqref{eqn:fadam_4} holds. Otherwise, for the case $\mnorm{\mu_{k,i}} < \infty$, if the L.H.S. in~\eqref{eqn:fadam_4} is positive, then~\eqref{eqn:fadam_4} holds for 
\begin{align}
    h \leq \frac{p_k}{q_k}, \label{hbd_1}
\end{align}
where we denote
\begin{align*}
    p_k & = \theta_k (1-\beta_{1k}) \mnorm{\nabla_i f(x_k)}^2 - \beta_{1k} \mnorm{\mu_{k,i} \nabla_i f(x_k)}, \\
    q_k & = \frac{L}{2} (1-\beta_{1k}) \left(\theta_k (1-\beta_{1k}) \mnorm{\nabla_i f(x_k)}^2 - \beta_{1k} \mnorm{\mu_{k,i} \nabla_i f(x_k)} \right) \\
    & + \frac{L}{2} \beta_{1k}^2 \mnorm{\mu_{k,i}}^2.
\end{align*}
Now, the L.H.S. in~\eqref{eqn:fadam_4} being positive is equivalent to
\begin{align}
     \theta_k > \frac{\beta_{1k}}{1 - \beta_{1k}} \frac{\mnorm{\mu_{k,i}}}{\mnorm{\nabla_i f(x_k)}}. \label{eqn:theta_k_bd}
\end{align}
Since $\mnorm{\mu_{k,i}} < \infty$, we choose $\beta_{1k} < \frac{1}{1 + \frac{\mnorm{\mu_{k,i}}}{\mnorm{\nabla_i f(x_k)}}}$ so that $\beta_{1k} \in (0,1)$ and $\frac{\beta_{1k}}{1 - \beta_{1k}} \frac{\mnorm{\mu_{k,i}}}{\mnorm{\nabla_i f(x_k)}} < 1$. Such a choice of $\beta_{1k}$ allows us to choose $\theta_k \in \left(\frac{\beta_{1k}}{1 - \beta_{1k}} \frac{\mnorm{\mu_{k,i}}}{\mnorm{\nabla_i f(x_k)}},1\right)$. So, for $\beta_{1k} < \frac{1}{1 + \frac{\mnorm{\mu_{k,i}}}{\mnorm{\nabla_i f(x_k)}}}$ and $\theta_k \in \left(\frac{\beta_{1k}}{1 - \beta_{1k}} \frac{\mnorm{\mu_{k,i}}}{\mnorm{\nabla_i f(x_k)}},1\right)$, we have that the L.H.S. in~\eqref{eqn:fadam_4} is positive. Moreover, since $\mnorm{\mu_{k,i}} < \infty$, both numerator and denominator in the R.H.S. of~\eqref{hbd_1} are $\bigO\left(\!\theta_k (1\!-\!\beta_{1k}) \mnorm{\nabla_i f(x_k)}^2\!-\!\beta_{1k} \mnorm{\mu_{k,i}\!\nabla_i f(\!x_k\!)}\!\right)$. So, the R.H.S. of~\eqref{hbd_1} is positive. So, we can choose $h > 0$ satisfying~\eqref{hbd_1} so that~\eqref{eqn:fadam_4} holds. We conclude that, $\exists \overline{\beta_k}, \underline{\theta_k} \in (0,1)$ such that for $\beta_{1k} \in [0,\overline{\beta_k})$ and $\theta_k \in (\underline{\theta_k},1)$ there exists $h_{1} \in (0,\infty)$ such that for $h < h_{1}$,~\eqref{eqn:fadam_4} holds, which implies that~\eqref{eqn:fadam_3} holds for $i \in I$.

Upon substituting from~\eqref{eqn:fadam_3} in~\eqref{eqn:fadam_2}, we have
\begin{align}
    & f(x_{k+1}) - f(x_k) \nonumber \\
    & \leq - h \sum_{i \in I} \frac{(1-\theta_k)(1-\beta_{1k}) (1-h\frac{L}{2}(1-\beta_{1k})) \mnorm{\nabla_i f(x_k)}^2}{\sqrt{\nu_{k+1,i}} + \epsilon} \nonumber \\
    & - h \sum_{i \in I'} \frac{(1-\beta_{1k}) (1-h\frac{L}{2}(1-\beta_{1k})) \mnorm{\nabla_i f(x_k)}^2}{\sqrt{\nu_{k+1,i}} + \epsilon} \nonumber \\
    & \leq - h \sum_{i \in I} \frac{(1-\theta_k)(1-\beta_{1k}) (1-h\frac{L}{2}(1-\beta_{1k})) \mnorm{\nabla_i f(x_k)}^2}{\sqrt{\nu_{k+1,i}} + \epsilon} \nonumber \\
    & - h \sum_{i \in I'} \frac{(1-\theta_k)(1-\beta_{1k}) (1-h\frac{L}{2}(1-\beta_{1k})) \mnorm{\nabla_i f(x_k)}^2}{\sqrt{\nu_{k+1,i}} + \epsilon} \nonumber \\
    & = - h (1-\theta_k)(1-\beta_{1k}) (1-h\frac{L}{2}(1-\beta_{1k})) \sum_{i=1}^d \frac{\mnorm{\nabla_i f(x_k)}^2}{\sqrt{\nu_{k+1,i}} + \epsilon} \nonumber \\
    & \leq 0. \label{eqn:fadam_5}
\end{align}
Under Assumption~\ref{assump_1}, we follow the argument in the proof of Theorem~\ref{thm:adagrad} after~\eqref{eqn:fdiff_2}, and obtain
\begin{align}
    \lim_{k \to \infty} \frac{\mnorm{\nabla_i f(x_k)}^2}{\sqrt{\nu_{k+1,i}} + \epsilon} = 0, \, \forall i. \label{eqn:seq_adam}
\end{align}
Under Assumption~\ref{assump_3}, following the argument in the proof of Theorem~\ref{thm:adagrad} after~\eqref{eqn:seq_adagrad}, we have $\lim_{k \to \infty} \nabla f(x_k) = 0_d$, and $\lim_{k \to \infty} f(x_k) = f_*$, and $\exists M \in (0,\infty)$ such that $\sqrt{\nu_{k+1,i}} + \epsilon \leq M \,  \forall i$. From~\eqref{eqn:fadam_5}, then we obtain
\begin{align*}
    & f(x_{k+1}) - f(x_k) \\
    & \leq - \frac{h}{M}(1-\theta_k)(1-\beta_{1k}) (1-h\frac{L}{2}(1-\beta_{1k})) \norm{\nabla f(x_k)}^2.
\end{align*}
Under Assumption~\ref{assump_3} and defining $c_{1k} = h (1-\theta_k)(1-\beta_{1k}) (1-h\frac{L}{2}(1-\beta_{1k})) \frac{2l}{M}$, from above we get
\begin{align}
    f(x_{k+1}) - f_* & \leq (1 - c_{1k}) (f(x_k) - f_*), \, \forall k \geq T. \label{eqn:conv_adam_1}
\end{align}
Moreover, $0 < h < \frac{2}{L(1-\beta_{1k})}$ implies that $c_{1k} > 0$ and $0 < h < \frac{M}{2l}$ implies that $c_1 < 1$. So, $(1-c_1) \in (0,1)$ for $0 < h < \min\{\frac{2}{L (1-\beta_{1k})}, \frac{M}{2l}, h_1\}$.

\noindent {\bf Case-II}: Next, we consider the case when $\nu_{k+1,i} \leq (1-\epsilon)^2$ for all $i \in \{1,\ldots,d\}$ for all $k \geq 0$. Then, $\sqrt{\nu_{k+1,i}} + \epsilon \leq 1$. Also, $\nu_0 = 0_d$ and~\eqref{eqn:nu_adam} implies that $\nu_{k,i} \geq 0 \, \forall k,i$.
From~\eqref{eqn:fadam_1}, then we have
\begin{align*}
    f(x_{k+1}) - f(x_k) & \leq \sum_{i=1}^d \left( - h \mu_{k+1,i} \nabla_i f(x_k) + \frac{L h^2}{2 \epsilon^2} \mnorm{\mu_{k+1,i}}^2\right).
\end{align*}
We note that the above inequality is similar to~\eqref{eqn:case_1} where $\sqrt{\nu_{k+1,i}} + \epsilon$ in~\eqref{eqn:case_1} is replaced by $1$ and $L$ is replaced by $\frac{L}{\epsilon^2}$. So, we follow the proof above in Case-I, and instead of~\eqref{eqn:fadam_5}, we obtain that
\begin{align}
    & f(x_{k+1}) - f(x_k) \nonumber \\
    & \leq - h (1-\theta_k)(1-\beta_{1k}) (1-h\frac{L}{2\epsilon^2}(1-\beta_{1k})) \sum_{i=1}^d \mnorm{\nabla_i f(x_k)}^2 \nonumber \\
    & = \hspace{-0.2em} - h (1-\theta_k)(1-\beta_{1k}) (1-h\frac{L}{2\epsilon^2}(1-\beta_{1k})) \norm{\nabla f(x_k)}^2. \label{eqn:fadam_6}
\end{align}
Under Assumption~\ref{assump_3} and defining $c_{2k} = h (1-\theta_k)(1-\beta_{1k}) (1-h\frac{L}{2\epsilon^2}(1-\beta_{1k})) 2l$, from above we get
\begin{align}
    f(x_{k+1}) - f_* & \leq (1 - c_{2k}) (f(x_k) - f_*), \, \forall k \geq 0. \label{eqn:conv_adam_2}
\end{align}

The rest of the proof follows the same argument as in the proof of Theorem~\ref{thm:adagrad}.
\end{proof}

\begin{remark}
    Bias correction: Since $\mu_0 = \nu_0 = 0$, the moment estimates $\mu_k$ and $\nu_k$ are biased towards zero at the early stages of iterations. Thus, in practice, Adam is implemented with a bias correction that accounts for this initialization of the moment estimates at zero~\cite{kingma2014adam}. When bias correction is active, $h$ in~\eqref{eqn:x_adam} is essentially replaced by $h \frac{\sqrt{1-\beta_2^{k+1}}}{1-\beta_{1k}^{k+1}}$. Since $\beta_{1k}, \beta_2 \in [0,1)$, we have $\lim_{k \to \infty} \frac{1-\beta_{1k}^{k+1}}{\sqrt{1-\beta_2^{k+1}}} = 1$ and $0 < \frac{1-\beta_{1k}^{k+1}}{\sqrt{1-\beta_2^{k+1}}} < \infty$. Thus, in the case of initial bias correction, the result presented in Theorem~\ref{thm:adam} is valid, upon replacing $\overline{h}$ with the positive quantity $\overline{h} \min_{k \geq 0} \frac{1-\beta_{1k}^{k+1}}{\sqrt{1-\beta_2^{k+1}}}$.    
\end{remark}

\begin{remark}
    Like our proof above, some of the existing analyses of Adam require a decreasing $\beta_{1k}$, including~\cite{chen2018convergence, kingma2014adam, reddi2019convergence}. While there are convergence results of Adam with constant $\beta_1$, they do not prove {\em linear} convergence.
\end{remark}

\begin{remark}
    We do not utilize the explicit update (3a) or (4b) of the denominator $y$ or $\nu$. Our proofs hold as long as $y_{k,i} > 0$ or $\nu_{k,i} > 0$. The specific form of~\eqref{eqn:y_adagrad} or~\eqref{eqn:nu_adam} only implicitly appears in our analysis in terms of its bound $M$. Therefore, Adam's analysis trivially extends to AdaBelief~\cite{zhuang2020adabelief} and AMSGrad~\cite{reddi2019convergence}. In RAdam~\cite{liu2019variance}, $ \exists T < \infty$ such that $\rho(k)$ is increasing for $k \geq T$ and converges to $\rho_{\infty}$. Then, by definition, $0 < r(k) < \infty$. Then, following the argument in Remark~1, our analysis of Adam easily extends to RAdam for $k \geq T$ by replacing $\overline{h}$ with $\overline{h} \min_{k \geq T} r(k)$. 
\end{remark}
\section{AdaGrad with stochastic gradients}

In practice, neural networks are trained with stochastic or mini-batch gradients at each iteration of the optimizer. So, in this section, we present our analysis of AdaGrad in the stochastic setting. For simplicity, we consider only one data point in each iteration. However, our result is applicable to mini-batch with more than one data points. For each data point $\mathcal{D}$, we define the loss function $l: \R^d \to \R$ as $l(\cdot, \mathcal{D})$ and its gradient $g(x;\mathcal{D}) = \nabla_x l(x;\mathcal{D})$. Our aim is to minimize the empirical risk,
\begin{align}
     \min_{x \in \R^d} \, f(x): = \E_{\zeta} \left[l(x;\mathcal{D}_{\zeta}) \right]. \label{eqn:opt_2}
\end{align}

At each iteration $k \geq 0$, AdaGrad randomly chooses a data point $\mathcal{D}_{\zeta_k}$, based on the realization $\zeta_k$ of the random variable $\zeta$, and computes its stochastic gradient $g_{\zeta_k}(k) = g(x_k;\mathcal{D}_{\zeta_k})$. We let $g_{i,\zeta_k}(k)$ denote the $i$-th element of $g_{\zeta_k}(k)$, for each $i \in \{1,\ldots,d\}$. Instead of~\eqref{eqn:y_adagrad}-\eqref{eqn:x_adagrad}, $x_k$ and $y_k$ are updated as
\begin{subequations}
\begin{align}
    x_{k+1,i} & = x_{k,i} - h \frac{g_{i,\zeta_k}(k)}{\sqrt{y_{k,i}} + \epsilon}, \label{eqn:x_st} \\
    y_{k+1,i} & = y_{k,i} + \mnorm{g_{i,\zeta_k}(k)}^2. \label{eqn:y_st}
\end{align}
\end{subequations}

For each iteration $k\geq 0$ we define the following.
\begin{itemize}
    \item Let $\condexp{\cdot}$ denote the conditional expectation of a function the random variables $\zeta_k$, given the current $x_k$ and $y_k$. 
    \item Let $\totexp{\cdot}$ denote the total expectation of a function of the random variables $\{\zeta_0,\ldots,\zeta_k\}$ given the initial condition $x_0$ and $y_0$. Specifically, $\totexp{\cdot} = \E_{\zeta_0,\ldots,\zeta_k}[\cdot], ~ k\geq 0$.
    \item For each $i \in \{1,\ldots,d\}$, define the conditional variance of $g_{i,\zeta_k}(k)$, which is a function of the random variable $\zeta_k$, given the current $x_k$ and $y_k$ as
    $\var{g_{i,\zeta_k}(k)} = \condexp{\mnorm{g_{i,\zeta_k}(k) - \condexp{g_{i,\zeta_k}(k)}}^2} = \condexp{\mnorm{g_{i,\zeta_k}(k)}^2} - \mnorm{\condexp{g_{i,\zeta_k}(k)}}^2$.
\end{itemize}

We make two additional standard assumptions~\cite{bottou2010large, wang2023convergence} for stochastic gradients as follows. Assumption~\ref{assump_5} is regarding boundedness of coordinate-wise affine noise variance~\cite{wang2023convergence}.

\begin{assumption} \label{assump_4}
At each iteration $k \geq 0$, the stochastic gradient is an unbiased estimate of true gradient, i.e., 
$\condexp{g_{\zeta_k}(k)} = \nabla f(x_k)$.
\end{assumption}

\begin{assumption} \label{assump_5}
For each $i \in \{1,\ldots,d\}$, there exist two non-negative real scalar values $V_{1,i}$ and $V_{2,i}$ such that, for each $k \geq 0$, 
\begin{align*}
    \var{g_{i,\zeta_k}(k)} \leq V_{1,i} + V_{2,i} \mnorm{\nabla_i f(x_k}^2.
\end{align*}
\end{assumption}

We define $M = \max_i V_{1,i}$ and $M_G = \max_i (V_{2,i}+1)$.

\begin{theorem} \label{thm:stoc}
Consider the AdaGrad algorithm in~\eqref{eqn:x_st}-\eqref{eqn:y_st} with initialization $y_0 = 0_d$ and the parameter $\epsilon \in (0,1)$. If Assumptions~\ref{assump_1}-\ref{assump_5} hold, then there exists $\omega \in (0,\infty)$ such that the following statements are true.
\begin{enumerate}[(i)]
    \item $\exists \overline{h} > 0$ such that if the step size $0 < h < \overline{h}$, then $\condexp{f(x_{k+1})} - f(x_k) \leq \rho (f(x_k) - f_*) + \omega, $ for all $k \geq 0$, where $\rho \in (0,1)$ and $\omega = \bigO(M)$.
    \item $\lim_{k \to \infty} \norm{\nabla f(x_k)}^2 > d \frac{L M h}{2 -L M_G h}$.
    \item Given arbitrary choices of the initial $x_0 \in \R^d$,
    \begin{align*}
    \lim_{k \to \infty} \totexp{f(x_{k+1})} \leq \frac{\omega}{1-\rho}.
\end{align*}
\end{enumerate}
\end{theorem}

\begin{proof}
Consider an arbitrary iteration $k \geq 0$. Under Assumption~\ref{assump_2}, from~\eqref{eqn:smooth} we have
\begin{align*}
    f(x_{k+1}) - f(x_k) & \leq (x_{k+1} - x_k)^{\top} \nabla f(x_k) + \frac{L}{2} \norm{x_{k+1} - x_k}^2.
\end{align*}
Upon taking conditional expectation on both sides and substituting above from~\eqref{eqn:x_st},
\begin{align}
    & \condexp{f(x_{k+1})} - f(x_k) \nonumber \\
    & \leq \sum_{i=1}^d \left( \condexp{\frac{-h g_{i,\zeta_k}(k)}{\sqrt{y_{k,i}} + \epsilon}} \nabla_i f(x_k) + \frac{L h^2}{2} \condexp{\frac{\mnorm{g_{i,\zeta_k}(k)}^2}{\mnorm{\sqrt{y_{k,i}} + \epsilon}^2}}\right) \nonumber \\
    & = \sum_{i=1}^d \left( -h \frac{\mnorm{\nabla_i f(x_k)}^2}{\sqrt{y_{k,i}} + \epsilon} + \frac{L h^2}{2} \frac{\condexp{ \mnorm{g_{i,\zeta_k}(k)}^2}}{\mnorm{\sqrt{y_{k,i}} + \epsilon}^2}\right),
    \label{eqn:fdiffst_1}
\end{align}
where the last equality follows from Assumption~\ref{assump_4}.
From the definition of conditional variance of $g_{i,\zeta_k}(k)$ and Assumptions~\ref{assump_4}-\ref{assump_5},
\begin{align*}
    \condexp{\mnorm{g_{i,\zeta_k}(k)}^2} & \leq V_{1,i} + (V_{2,i} + 1) \mnorm{\nabla_i f(x_k}^2 \\
    & \leq M + M_G \mnorm{\nabla_i f(x_k)}^2.
\end{align*}
Upon substituting from above in~\eqref{eqn:fdiffst_1},
\begin{align}
    & \condexp{f(x_{k+1})} - f(x_k) \nonumber \\
    & \leq \sum_{i=1}^d \left( -h \frac{\mnorm{\nabla_i f(x_k)}^2}{\sqrt{y_{k,i}} + \epsilon} + \frac{L h^2}{2} \frac{M + M_G \mnorm{\nabla_i f(x_k)}^2}{\mnorm{\sqrt{y_{k,i}} + \epsilon}^2}\right). \label{eqn:fdiffst_2}
\end{align}

\noindent {\bf Case-I}: First, we consider the case when $y_{k,i} > (1-\epsilon)^2$ for all $i \in \{1,\ldots,d\}$ for all $k \geq T$, where $T < \infty$. Consider any iteration $k \geq T$. Then, $\mnorm{\sqrt{y_{k,i}} + \epsilon}^2 > \sqrt{y_{k,i}} + \epsilon$. From~\eqref{eqn:fdiffst_2}, then we have
\begin{align}
    & \condexp{f(x_{k+1})} - f(x_k) \nonumber \\
    & \leq -h \sum_{i=1}^d \frac{(1-\frac{L M_G}{2} h) \mnorm{\nabla_i f(x_k)}^2 - \frac{LM}{2} h}{\sqrt{y_{k,i}} + \epsilon}. \label{eqn:fdiffst_3}
\end{align}
For $h < \frac{2}{L M_G}$, we have $(1 - \frac{L M_G}{2} h) > 0$. 
Now, except at the trivial point $\nabla f(x_k) = 0_d$, there exists at least one $j \in \{1,\ldots,d\}$ for which $\nabla_j f(x_t) \neq 0$. We consider the non-empty set $I = \{i \, | i\in \{1,\ldots,d\}, \nabla_i f(x_k) \neq 0 \}$ and its complement $I'$. Then, we can rewrite the R.H.S. above as
\begin{align*}
    & \sum_{i=1}^d \frac{(1-\frac{L M_G}{2} h) \mnorm{\nabla_i f(x_k)}^2 - \frac{LM}{2} h}{\sqrt{y_{k,i}} + \epsilon} \\
    & = \sum_{i \in I} \frac{(1-\frac{L M_G}{2} h) \mnorm{\nabla_i f(x_k)}^2 - \frac{LM}{2} h}{\sqrt{y_{k,i}} + \epsilon} + \sum_{j \in I'} \frac{- \frac{LM}{2} h}{\sqrt{y_{k,j}} + \epsilon}.
\end{align*}
Upon rearranging the terms above, we have
\begin{align}
    & \sum_{i=1}^d \frac{(1-\frac{L M_G}{2} h) \mnorm{\nabla_i f(x_k)}^2 - \frac{LM}{2} h}{\sqrt{y_{k,i}} + \epsilon} > 0 \nonumber \\
    & \iff h < \frac{2}{L} \frac{\sum_{i \in I} \frac{\mnorm{\nabla_i f(x_k)}^2}{\sqrt{y_{k,i}} + \epsilon}}{\sum_{i \in I} \frac{M_G \mnorm{\nabla_i f(x_k)}^2 + M}{\sqrt{y_{k,i}} + \epsilon} + \sum_{j \in I'} \frac{M}{\sqrt{y_{k,j}} + \epsilon}}. \label{eqn:hst_1}
\end{align}
Since $\sum_{j \in I'} \frac{M}{\sqrt{y_{k,j}} + \epsilon} < \infty$, both the numerator and denominator in the R.H.S. of~\eqref{eqn:hst_1} are $\bigO \left(\sum_{i \in I} \frac{\mnorm{\nabla_i f(x_k)}^2}{\sqrt{y_{k,i}} + \epsilon}\right)$. So, the R.H.S. of~\eqref{eqn:hst_1} is positive. Then, $h > 0$ can be chosen to satisfy~\eqref{eqn:hst_1}, for which 
\begin{align}
    \sum_{i=1}^d \frac{(1-\frac{L M_G}{2} h) \mnorm{\nabla_i f(x_k)}^2 - \frac{LM}{2} h}{\sqrt{y_{k,i}} + \epsilon} > 0. \label{eqn:rhs_h}
\end{align}
Then,~\eqref{eqn:fdiffst_3} implies that $\condexp{f(x_{k+1})} < f(x_k)$. Since $\{\condexp{f(x_{k+1})}: k\geq T\}$ is a decreasing sequence and, under Assumption~\ref{assump_1}, bounded from below by $f_*$, the sequence converges with $\lim_{k \to \infty} \condexp{f(x_{k+1})} < \infty$. Next, upon summation from $t=T$ to $t=k$ on both sides of~\eqref{eqn:fdiffst_3}, due to telescopic cancellation, we obtain the total expectation
\begin{align*}
    & \totexp{f(x_{k+1})} - f(x_T) \\
    & \leq - h \sum_{t=T}^k \sum_{i=1}^d \frac{(1-\frac{L M_G}{2} h) \mnorm{\nabla_i f(x_t)}^2 - \frac{LM}{2} h}{\sqrt{y_{t,i}} + \epsilon}.
\end{align*}
Combining with~\eqref{eqn:rhs_h}, we have
\begin{align*}
    & 0 < \lim_{k \to \infty} \sum_{t=T}^k \sum_{i=1}^d \frac{(1-\frac{L M_G}{2} h) \mnorm{\nabla_i f(x_t)}^2 - \frac{LM}{2} h}{\sqrt{y_{t,i}} + \epsilon} \\
    & \leq \lim_{k \to \infty} (f(x_T) - \totexp{f(x_{k+1})}) \leq f(x_T) - f_*. 
\end{align*}
Following the argument as in the proof of Theorem~\ref{thm:adagrad}, the above implies that
\begin{align}
    \lim_{k \to \infty} \frac{(1-\frac{L M_G}{2} h) \mnorm{\nabla_i f(x_k)}^2 - \frac{LM}{2} h}{\sqrt{y_{k,i}} + \epsilon} = 0, \, \forall i. \label{eqn:seq_st}
\end{align}
So, either $\lim_{k \to \infty} (1-\frac{L M_G}{2} h) \mnorm{\nabla_i f(x_k)}^2 - \frac{LM}{2} h = 0$ or $\lim_{k \to \infty} \sqrt{y_{k,i}} = \infty$. Consider the case $\lim_{k \to \infty} \sqrt{y_{k,i}} = \infty$. From~\eqref{eqn:seq_st}, we have $(1-\frac{L M_G}{2} h) \mnorm{\nabla_i f(x_k)}^2 - \frac{LM}{2} h = \smallO(\sqrt{y_{k,i}} + \epsilon)$, which implies $\frac{\sqrt{\mnorm{(1-\frac{L M_G}{2} h) \mnorm{\nabla_i f(x_k)}^2 - \frac{LM}{2} h}}}{\sqrt{y_{k,i}} + \epsilon} = \smallO((\sqrt{y_{k,i}} + \epsilon)^{-0.5})$. So, if $\lim_{k \to \infty} \sqrt{y_{k,i}} = \infty$, we have $\lim_{k \to \infty} \frac{\sqrt{\mnorm{(1-\frac{L M_G}{2} h) \mnorm{\nabla_i f(x_k)}^2 - \frac{LM}{2} h}}}{\sqrt{y_{k,i}} + \epsilon} = 0$. Following the argument after~\eqref{eqn:seq_adagrad} in the proof of Theorem~\ref{thm:adagrad}, $\lim_{k \to \infty} \sqrt{y_{k,i}} = \infty$ implies that $\lim_{k \to \infty} (1-\frac{L M_G}{2} h) \mnorm{\nabla_i f(x_k)}^2 - \frac{LM}{2} h = 0$. So, in either of the cases above, we have 
\begin{align}
    \lim_{k \to \infty} \mnorm{\nabla_i f(x_k)}^2 = \frac{\frac{LM}{2} h}{(1-\frac{L M_G}{2} h)} =: g_*, \, \forall i. \label{eqn:grad_star}
\end{align}
From the above argument, it further follows that $\lim_{k \to \infty} \sqrt{y_{k,i}} = \infty$ is possible only if $\lim_{k \to \infty} \norm{\nabla f(x_k)}^2 = d g_*$. For $\lim_{k \to \infty} \mnorm{\nabla f(x_k)}^2 > d g_*$, then $\sqrt{y_{k,i}} + \epsilon$ is bounded. Then, for $0 < h < \frac{2}{L M_G}$, $h$ satisfying~\eqref{eqn:rhs_h}, and for $\lim_{k \to \infty} \mnorm{\nabla f(x_k)}^2 > d g_*$, $\exists M \in (0,\infty)$ such that $\sqrt{y_{k,i}} + \epsilon \leq B \, \forall i$. Then, from~\eqref{eqn:fdiffst_3},
\begin{align*}
    & \condexp{f(x_{k+1})} - f(x_k) \\
    & \leq - \frac{h}{B} (1 - \frac{L M_G}{2} h ) \norm{\nabla f(x_k)}^2 + \frac{L M d}{2B} h^2.
\end{align*}
Under Assumption~\ref{assump_3}, from above we have
\begin{align*}
    & \condexp{f(x_{k+1})} - f(x_k) \\
    & \leq - \frac{h}{B} \left(1 - \frac{L M_G}{2} h\right) 2l (f(x_k) - f_*) + \frac{L M d}{2B} h^2.
\end{align*}
Upon defining $c_1 = \frac{h}{B} \left(1 - \frac{L M_G}{2} h\right) 2l$ and $\omega = \frac{L M d}{2B} h^2$, we rewrite the above as
\begin{align}
    \condexp{f(x_{k+1})} - f_* & \leq (1 - c_1) (f(x_k) - f_*) + \omega, \, \forall k \geq T. \label{eqn:conv_st_1}
\end{align}
Moreover, $0 < h < \frac{2}{L M_G}$ implies that $c_1 > 0$ and $0 < h < \frac{B}{2l}$ implies that $c_1 < 1$. So, $(1-c_1) \in (0,1)$ for $0 < h < \min\{\frac{2}{L M_G}, \frac{B}{2l}\}$.

\noindent {\bf Case-II}: Next, we consider the case when $y_{k,i} \leq (1-\epsilon)^2$ for all $i \in \{1,\ldots,d\}$ for all $k \geq 0$. Then, $\sqrt{y_{k,i}} + \epsilon \leq 1$. Also, $y_0 = 0_d$ and~\eqref{eqn:y_st} implies that $y_{k,i} \geq 0 \, \forall k,i$.
From~\eqref{eqn:fdiffst_2}, then we have
\begin{align*}
    & \condexp{f(x_{k+1})} - f(x_k) \\
    & = - h \left(1 - \frac{L M_G h}{2\epsilon^2}\right) \norm{\nabla f(x_k)^2} + \frac{L M d h^2}{2 \epsilon^2}.
\end{align*}
Under Assumption~\ref{assump_3}, from above we have
\begin{align*}
    & \condexp{f(x_{k+1})} - f(x_k) \\
    & \leq - h \left(1 - \frac{L M_G h}{2\epsilon^2}\right) 2l (f(x_k) - f_*) + \frac{L M d h^2}{2 \epsilon^2}.
\end{align*}
Upon defining $c_2 = h \left(1 - \frac{L M_G h}{2 \epsilon^2}\right) 2l$, we obtain that
\begin{align}
    \condexp{f(x_{k+1})} - f_* & \leq (1 - c_2) (f(x_k) - f_*) + \omega. \label{eqn:conv_st_2}
\end{align}
Moreover, $(1-c_2) \in (0,1)$ for $0 < h < \min\left\{\frac{2\epsilon^2}{L M_G}, \frac{1}{2l}\right\}$. 

Following the same argument as in the proof of Theorem~\ref{thm:adagrad},
for $0 < h < \min\{\frac{2 \epsilon^2}{L M_G}, \frac{M}{2l}, \frac{1}{2l}\}$, $h$ satisfying~\eqref{eqn:rhs_h}, and for $\lim_{k \to \infty} \norm{\nabla f(x_k)}^2 > d g_*$,
\begin{align*}
    & \condexp{f(x_{k+1})} - f(x_k) \leq \rho (f(x_k) - f_*) + \omega, \, \forall k \geq 0,
\end{align*}
where $\rho = \max\{(1 - c_1),(1-c_2)\}$. Since $\rho \in (0,1)$, upon iterating the above from $k$ to $0$, by the law of total expectation, $\lim_{k \to \infty} \totexp{f(x_{k+1})} \leq \frac{\omega}{1-\rho}$. The proof is complete.
\end{proof}

According to Theorem~\ref{thm:stoc}, AdaGrad in~\eqref{eqn:x_st}-\eqref{eqn:y_st} converges {\em linearly} in expectation to a neighborhood of minima $f_*$, for small enough stepsize $h$. The neighborhood of $f_*$ to which $\totexp{f(x_{k+1}}$ converges is $\bigO(M)$, i.e., proportional to the variance of stochastic gradients evaluated at the minimum point. Furthermore, the gradient-norm $\norm{\nabla f(x_k)}^2$ converges to a limit greater than a value of $\bigO(M)$. 

\begin{remark}
    Following the steps in Theorem~\ref{thm:adam} and Theorem~\ref{thm:stoc}, {\em linear} convergence in expectation to a neighborhood of minima $f_*$ can be proved for Adam with stochastic gradients, i.e., $\lim_{k \to \infty} \totexp{f(x_{k+1})} \leq \frac{\omega}{1-\rho}$ and $\lim_{k \to \infty} \norm{\nabla f(x_k)}^2 > 0$. It implies that Adam with stochastic gradients not necessarily converges to the minima. This implication of our analysis is consistent with the result in~\cite{reddi2019convergence} that Adam has non-zero regret at $k \to \infty$ for some online optimization problems.
\end{remark}
\section{Summary}

We presented a framework that proves {\em linear} convergence of two adaptive gradient methods, namely AdaGrad and Adam in discrete-time, for minimizing smooth objective functions that satisfy the PL inequality. Among the prior works on adaptive gradient methods, only the AdaGrad-Norm algorithm and a continuous-time version of Adam have provable {\em linear} convergence, for a class of optimization problems. Thus, our work contributes towards reducing the theoretical gap between vanilla gradient-descent and the more successful adaptive gradient optimizers. The unifying approach in our framework could be applicable in rigorously analyzing other adaptive gradient optimizers.






\bibliography{mybibfile}

\end{document}